\documentclass[10pt]{article}
\usepackage{comment} 
\usepackage{lipsum} 
\usepackage[lmargin=1in,rmargin=1in,bmargin=1.25in,tmargin=1.25in]{geometry}

\usepackage[utf8]{inputenc} 
\usepackage[T1]{fontenc}    

\usepackage{amsmath, mathtools, amsthm, amssymb}
\usepackage{amsfonts}       
\usepackage{mathrsfs}
\usepackage{enumitem}
\usepackage{clrscode3e}
\usepackage[ruled,vlined]{algorithm2e}
\usepackage{thmtools} 
\usepackage{cancel}
\usepackage{physics}
\usepackage{bm}
\usepackage{bbm}
\usepackage{dsfont}
\usepackage{xfrac}

\usepackage[numbers]{natbib}
\usepackage{graphicx}
\usepackage[dvipsnames]{xcolor} 
\usepackage{booktabs} 
\usepackage[allbordercolors=blue,hidelinks]{hyperref}

\usepackage{tikz}
\usepackage{pgfplots}
\usepackage{wrapfig}
\usepackage{float}
\usepackage{subfigure}

\usepackage{titlesec} 
\usepackage[margin=1.5cm, font=small, labelfont=bf,labelsep=period]{caption}

\usepackage{import}

\usepackage{doi}

\usepackage{hyperref}       
\usepackage[symbol]{footmisc}

\usepackage{hyperref}       
\usepackage{url}            
\usepackage{booktabs}       
\usepackage{amsfonts}       
\usepackage{nicefrac}       
\usepackage{microtype}      

\usepackage{graphicx}
\usepackage{subfigure}
\usepackage{booktabs}
\usepackage{physics}
\usepackage{amsmath}
\usepackage{bbm}
\usepackage{mathtools}
\usepackage{dsfont}
\usepackage{amssymb}
\usepackage{xfrac}
\usepackage{tikz}
\usepackage{pgfplots}
\usepackage{float}
\usepackage{wrapfig}

\usepackage{doi}

\usepackage{xr}

\graphicspath{{figures/}}

\usepackage{caption}

\usetikzlibrary{positioning,backgrounds,shapes.geometric,patterns,arrows,backgrounds}
\usepgfplotslibrary{ternary}
\pgfplotsset{
    compat=1.16,
    /pgf/declare function={
        binom(\n,\k) = int(round(round(\n)! / (round(\n - \k)! * round(\k)!)));
    },
}
\DeclareMathOperator{\sign}{sgn}
\DeclareMathOperator{\dist}{dist}

\newcommand\latin[1]{\textit{#1}}

\let\savedS=\S

\newcommand{\N}{\mathbb{N}}
\renewcommand{\P}{\mathbb{P}}

\newcommand{\Z}{\mathbb{Z}}
\usepackage{bbm}

\newcommand{\A}{\mathcal{A}}

\newcommand{\K}{\mathcal{K}}
\renewcommand{\L}{\mathcal{L}}

\renewcommand{\S}{\mathcal{S}}

\newcommand{\X}{\mathcal{X}}
\newcommand{\Y}{\mathcal{Y}}

\newcommand{\fa}{\forall}

\newcommand{\ceq}{\coloneqq} 

\DeclarePairedDelimiterX{\inp}[2]{\langle}{\rangle}{#1, #2} 

\newcommand{\bpar}[1]{\left(#1\right)}

\newcommand{\no}{\noindent}

\usepackage{calligra}
    \DeclareMathAlphabet{\mathcalligra}{T1}{calligra}{m}{n}
    \DeclareFontShape{T1}{calligra}{m}{n}{<->s*[2.2]callig15}{}

\font\bigbold=cmbx12

\font\midheader=cmssbx11 
\font\smallheader=cmssbx10 

\def\maketitle#1#2#3#4#5{
  \centerline {\bigbold #1}
  \medskip
  \centerline {\eightpt #2} 
  \medskip
  \centerline {\tensc #3}
  \centerline {\tensc #4}
  \medskip
  \centerline {\sl #5}
  \bigskip
}

\newcounter{daggerfootnote}

\titleformat{\section}{\midheader}{\thesection.}{0.5em}{}
\titlespacing{\section}{0em}{1.5\bigskipamount}{\medskipamount}
\titleformat{\subsection}{\smallheader}{\thesubsection.}{0.5em}{}
\titlespacing{\subsection}{0em}{\bigskipamount}{\medskipamount}
\titleformat{\paragraph}[runin]{\bf}{}{0em}{}[.]
\titlespacing{\paragraph}{0em}{\medskipamount}{*0.75}

\font\tensc=cmcsc10

\font\eightpt=cmr8

\setlist[enumerate]{itemsep=\smallskipamount,parsep=0pt}
\setlist[itemize]{itemsep=\smallskipamount,parsep=0pt}

\renewcommand{\theequation}{\oldstylenums{\arabic{equation}}} 

\newtheoremstyle{nicertheorem}
{3pt}
{3pt}
{\sl}
{}
{\bf}
{.}
{.5em}
{}

\theoremstyle{nicertheorem}
\newtheorem{thm}{Theorem}
\newtheorem{cor}[thm]{Corollary}

\theoremstyle{definition}

\theoremstyle{remark}

\renewcommand{\mathbb}{\mathbf}

\graphicspath{{figures/}}

\renewcommand{\P}{\mathcal{P}}
\newcommand{\res}{\mathrm{res}} 

\DeclarePairedDelimiter\floor{\lfloor}{\rfloor}
\renewcommand{\ceq}{\coloneqq}

\begin{document}

\maketitle{A Study of Policy Gradient on a Class of Exactly Solvable Models}{}
{Gavin McCracken\footnote[1]{\href{mailto:gavin.mccracken@mail.mcgill.ca}{gavin.mccracken@mail.mcgill.ca}}\textsuperscript{\dag,\ddag}, Colin Daniels, Rosie Zhao\textsuperscript{\dag}, Anna Brandenberger\textsuperscript{\dag},
}
{Prakash Panangaden\textsuperscript{\dag,\ddag} and Doina Precup\textsuperscript{\dag,\ddag,\savedS}}
{\textsuperscript{\dag}School of Computer Science, McGill University, 
\textsuperscript{\ddag}MILA, 
\textsuperscript{\savedS}DeepMind
} 

\medskip

\[
  \vbox{
    \hsize 5.5 true in
    \noindent{\bf Abstract.}\enskip Policy gradient methods are extensively used in reinforcement learning as a way to optimize expected return. In this paper, we explore the evolution of the policy parameters, for a special class of exactly solvable \textsc{pomdp}s, as a continuous-state Markov chain, whose transition probabilities are determined by the gradient of the distribution of the policy's value. Our approach relies heavily on random walk theory, specifically on affine Weyl groups. We construct a class of novel partially observable environments with controllable exploration difficulty, in which the value distribution, and hence the policy parameter evolution, can be derived analytically. Using these environments, we analyze the probabilistic convergence of policy gradient to different local maxima of the value function. To our knowledge, this is the first approach developed to analytically compute the landscape of policy gradient in \textsc{pomdp}s for a class of such environments, leading to interesting insights into the difficulty of this problem. 
    \smallskip 
    
    \no \textbf{Keywords.}\enskip Policy gradient, \textsc{pomdp}, reinforcement learning, explicit value distribution.
  }
\]

\smallskip

\section{Introduction}
Policy gradient (PG) is a fundamental tool in reinforcement learning (RL) based
on changing the parameters of a policy to improve and ultimately maximize the
expected long-term
return~\cite{suttonPolicyGradientMethods2000,williamsSimpleStatisticalGradientfollowing1992,suttonReinforcementLearningIntroduction2018}. This
approach has led to several important
algorithms~\cite{lillicrapContinuousControlDeep2019,
  schulmanTrustRegionPolicy2017,espeholtIMPALAScalableDistributed2018}, and has
recently been used to develop frameworks in the continuous control
setting~\cite{barth-maronDistributedDistributionalDeterministic2018,
  tesslerDistributionalPolicyOptimization2019}. Theoretical analyses of PG
mainly rely on general properties of gradient
descent~\cite{guInterpolatedPolicyGradient2017,
  zhaoAnalysisImprovementPolicy2011, fazelGlobalConvergencePolicy2019}, and are
often complemented by empirical results relying on simulations. However,
empirical analyses can suffer from high variance, leading to impossibly large
computational costs in order to fully understand a non-trivial environment.

In our work we develop a new idea: the use of \emph{exact analytical solutions}
for the value distribution to obtain the gradient.  Clearly, one cannot obtain
such exact solutions for every model: we are not proposing a new ``method'' for
solving \textsc{pomdp}s or \textsc{mdp}s.  Rather, we are analyzing some special examples from
which one can extract more information because we can obtain exact analytical
forms by exploiting symmetry.

The importance of exactly solved models in physics, mathematical biology and
other areas of science has long been understood; see, for example the classic
book by Baxter~\cite{Baxter82} which has led to insights into phase transitions
not available by running simulations.

The importance of an exact solution stems from the fact that the exact
functional form gives insight that mere plots constructed from simulations can
only hint at.  In particular, one can learn the \emph{real} asymptotic form of
the behaviour and not guess from an unjustified extrapolation of the graphs.
Another important point is that one may be able to calculate quantities of
interest more efficiently; though this is not always guaranteed.  Thirdly, an
exact solution gives a much better handle on the accuracy of approximations
rather than the use of generic inequalities.  Finally exact solutions,
especially of nonlinear problems, show features that are totally hidden in
approximations.  Examples abound in geometry, differential equations and
physics.

The main critique of the use of exact solutions is that they can only be
obtained in situations that are too simple to be realistic.  Of course, exact
solutions are rare and difficult to obtain, but the effort invested in finding
them is well worth it.  As a counter to the ``too simple'' objection, we claim
that exact solutions usually contain size parameters that give a much better
handle on how problems really scale than numerical explorations can possibly
give.  In this paper the family of \textsc{pomdp}s, while, of course, not general is
still wide enough to be of interest.  There is no reason to think that the
qualitative features of the solutions obtained are special to these
environments.  Indeed, simulation studies are just as much tied to the specific
example being studied and the use of ``randomly seeded environments'' is
hardly a guarantee that the results are of generic significance.


In the present work we construct a family of partially observable environments
with controllable exploration difficulty, that were constructed from the
geometry of Weyl chambers: lattice structures that arise in the representation
theory of Lie algebras.  Their provenance from a branch of mathematics devoted
to symmetry is the key to the exact solvability of the problems in which we are
interested.  We adapt some known lattice-path counting techniques and apply them
to these structures.  We use the results to obtain exact expressions for the
value distribution function and explicit analytic formulas for the policy
parameter evolution.  Using these environments, we show that our approach allows
us to analyze the probabilistic convergence of policy gradient to different
local maxima of the value function.  To our knowledge, this is the first
approach developed to analytically compute the landscape of policy gradient in a
large class of \textsc{pomdp}s, leading to interesting insights into the difficulty of
this problem.  Indeed, we demonstrate how this framework can be used for calculating properties which are highly relevant, yet impossible to obtain from empirical approaches, such as the probability that an agent learns the optimal policy as a function of the learning rate (see Fig.~\ref{fig:1d-convergence-vs-step-size}), as well as the distributional evolution of the parameters through time (see Fig.~\ref{fig:distributional-evolution}).

Two notable features of our approach are: (i) the use of a higher-order
viewpoint on the problem and (ii) the explicit use of a distributional RL
approach.  The importance of (i) is that we view the RL algorithm itself as a
discrete-time Markov process on the continuous state space of value distribution
functions.  Thus the rich theory of Markov processes, specifically convergence
theory, can be brought to bear on the RL problem.  The point of (ii) is that we
can glean insights into distributional RL.



\begin{figure*}[htb]
    \hspace*{-.27cm}
    \centering
    \input{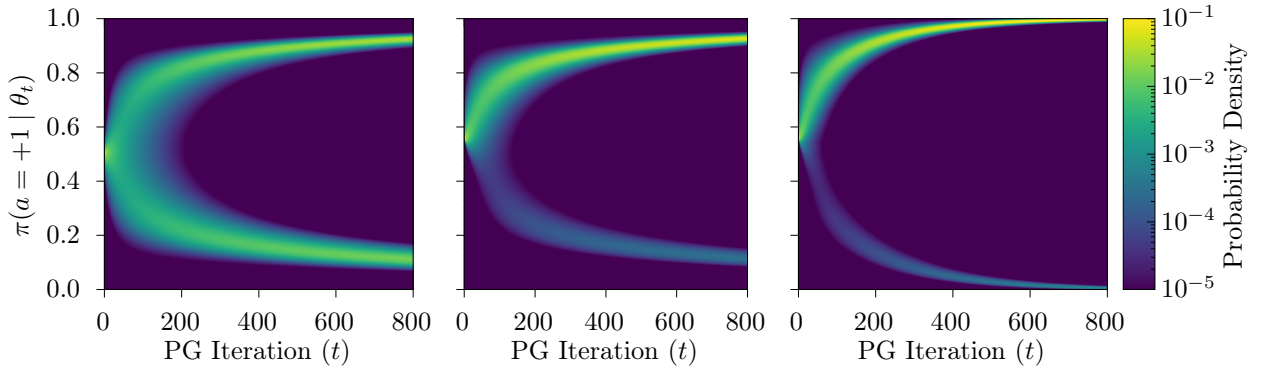}
    \vskip -\baselineskip
    \caption{
        By obtaining the exact analytical form of the value distribution, we can gain significant insight into the dynamics of learning. Above shows the distributional evolution of a single-parameter policy as it learns via policy gradient on a 10-state 1D gambler's ruin environment (see Fig.~\ref{fig:gamblers}). Specifically we see the consequence of starting with a policy at $t=0$ that favors taking a positive step $\pi(a = +1 \mid \theta_0) =  0.56$ (middle) vs. a purely random policy (left), as well as the effect of changing the optimizer to SGD with momentum, $momentum = 0.2$ (right) on the \textit{distribution} of learned policies at iteration t.
        For SGD with momentum, note that the probability transitions become  $\mathcal{P}_{\{v,\theta\},\{v',\theta'\}}$--- twice the dimensionality of \eqref{pg:transitions}. For numerical calculation details and specifics see Appendix~\ref{app:distributional-evolution}.
    }
    \vskip -\baselineskip
    \label{fig:distributional-evolution}
\end{figure*}

\section{Preliminaries}
The standard RL setting is an agent interacting with an environment, which can be modelled as either a Markov decision process (\textsc{mdp}) or, when the agent does not have access to complete state information, a partially observable Markov decision process (\textsc{pomdp}). A \textsc{mdp} is a 5-tuple: $(\S, \mathcal{A}, \mathcal{P}, \mathcal{R}, \gamma)$, where $S$ and $\mathcal{A}$ denote the state and action spaces assumed finite in this paper, $\mathcal{P}: \S \times \mathcal{A} \to \dist(\S)$ is the transition function, with $\P(s,a)(s')$ the probability of transitioning from $s\in \S$ to $s'\in \S$ given action $a \in \mathcal{A}$, $\mathcal{R}: S \times A \to [0,1]$ is the reward function, which we consider a random variable, and $\gamma \in [0,1]$ is the discount factor. \textsc{pomdp}s require two more components: $\omega$, the set of observations, and the observation probability distribution, $\mathcal{O}:\S \to \dist(\omega)$. Policies $\pi_{\vec{\theta}}$ map observations to distributions over actions: $\pi_{\vec{\theta}}: \mathcal{O} \to \mathcal{A}$, where $\vec{\theta}$ is the parameter vector. 

We will follow the distributional RL literature and consider the return random variable $Z^\pi: \S \times \A \to \mathrm{dist}(\mathcal{R})$, known as the value distribution~\cite{bellemareDistributionalPerspectiveReinforcement2017} or return distribution function~\cite{rowlandAnalysisCategoricalDistributional2018}, which is the sum of discounted rewards along the agent’s trajectory  when following the policy $\pi$. It satisfies the distributional version of Bellman's equations:
\begin{equation*}
Z^\pi(s, a) \stackrel{D}{=} R(s,a) + \gamma Z^\pi(s_{t+1}, s_{t+1}) = \sum_{t=0}^\infty \gamma^t R(s_t, a_t),
\end{equation*}
where $s_0 = s, a_0 = a$, the next state  $s_{t+1} \sim \P(s_t, a_t)$ and the trajectory is on-policy: $a_t \sim \pi(s_t) \; \fa t$.

The PG theorem quantifies how the return received by following a policy is changed as a result of changing the policies parameters, and gives the following update rule:
\begin{equation}
        \vec{\theta}_{t + 1} = \vec{\theta}_t + \alpha \nabla J(\vec{\theta}_t),
\end{equation}
\latin{s.t.}, under the \textsc{reinforce} update rule~\cite{williamsSimpleStatisticalGradientfollowing1992}, the gradient of policy value $J(\vec{\theta})$ is derived as~\cite{suttonPolicyGradientMethods2000, suttonReinforcementLearningIntroduction2018}
\begin{equation}
    \label{eq:grad_j}
    \begin{split}
          \nabla_{\vec{\theta}} J =\mathbb{E}_{\pi}\left[G_{t} \frac{\nabla \pi\left(A_{t} | S_{t}, \vec{\theta}\right)}{\pi\left(A_{t} | S_{t}, \vec{\theta}\right)}\right] 
          = \frac{1}{N} \sum_{i=1}^N \left(\sum_{t=1}^T \nabla_{\vec{\theta}} \log \pi(a_{i,t}|s_{i,t}, {\vec{\theta}}) \right) G,
    \end{split}
\end{equation}
where $G$ is the episodic return, $N$ is the batch-size, $\alpha$ is the learning rate and $T$ is the episode length.

\section{Analyzing policy gradient using Markov chain theory}
The main idea we propose is based on the observation that during policy gradient the parameter vector can be seen as evolving according to a Markov chain, where the next parameter vector considered depends only on the current one. Furthermore, by considering $\vec{\theta}_t$ a random variable we can write a distributional update for PG. Putting together the value distribution with the PG update:
\begin{gather}
    \vec{\theta}_{t + 1} \stackrel{D}{=} \vec{\theta}_t + \alpha \nabla J_t(\vec{\theta}_t), \\
    \nabla J_t(\vec{\theta}) \stackrel{D}{=} Z_t^{\pi}(S_0) \frac{\nabla \pi\big(A_{\tau} | S_{\tau}, \vec{\theta}\big)}{\pi\big(A_{\tau} | S_{\tau}, \vec{\theta}\big)} 
\end{gather} 
where $S_0\! = s_0$, $A_{\tau} \sim \pi_\theta(S_\tau)$, $S_{\tau + 1} \sim \P(S_\tau, A_\tau)$
and $Z_t^\pi(S_0) \stackrel{D}{=} Z^\pi(S_0, \pi_{\vec{\theta}_t}(S_0)) \stackrel{D}{=} \sum_{k=0}^T \gamma^k R(S_{k}, A_{k})$, with $T$ being the length of a trajectory. 

Using the resulting distribution of the policy gradients, we could 
then calculate and study PG by using the transition probabilities $\mathcal{P}_{\vec{\theta}, \vec{\theta}'}$, from some set of parameters $\vec{\theta}$ to another set $\vec{\theta}'$:
\begin{equation}
    \label{pg:transitions}
    \mathcal{P}_{\vec{\theta}, \vec{\theta}'} = \Pr{\alpha \nabla J(\vec{\theta}) \!= \vec{\theta}' - \vec{\theta}}.
\end{equation}
\-\ \ \ \ This is not the approach typically used in analyzing PG, which is often done using limit behavior (see also Sec.~\ref{sec:related}), but if we could implement this idea we could utilize tools from Markov Chain theory to obtain more extensive insight on PG. For example, we could study the distributional evolution of policies found by an algorithm as a function of hyperparameters like the learning rate, or choice of optimizer (see Fig.~\ref{fig:distributional-evolution}). We could also study in principle how the parameter vector should be structured as a function of specific properties of the environment. Tools from the probabilistic analysis of algorithms literature can also be used to understand the solution landscape of policy gradient and the likelihoods of converging to different maxima~\cite{devroyeProbabilisticTheoryPattern1996, levinMarkovChainsMixing2008}. Some of these ideas are proving quite useful in analyzing deep learning algorithms~\cite{saxeExactSolutionsNonlinear2014, linWhyDoesDeep2017}.

There are certain obstacles to this type of analysis which need to be overcome: it depends on our ability to compute the value distribution for a specific environment, and our ability to instantiate the gradient computation, which is a function of the PG algorithm. With a computationally tractable value distribution, the latter is easier, but is not trivial (we give an example for the popular \textsc{reinforce} algorithm). Computing the value distribution is particularly difficult because it requires considering all possible trajectories and computing their probabilities. We thus look into designing environments with this kind of counting in mind. 

Environment design is already done in many RL scenarios. For example, most deep RL works use ``intuitively'' difficult environments (\latin{i.e.}, arcade environments~\cite{bellemareArcadeLearningEnvironment2013} or Mujoco physics simulations~\cite{todorovMuJoCoPhysicsEngine2012}). Some recent efforts also look at environments designed with the goal of emphasizing a problem in RL (\latin{e.g.}, exploration) and produce families of environments where this aspect can be varied systematically~\cite{osbandBehaviourSuiteReinforcement2020}.

Our goal is to design a parametric family of environments where, at an intuitive level the structure is clear, but where some features of interest can be controlled by varying the parameters defining the members of the family.  Thus we can control the difficulty and complexity and study in explicit fashion the properties of policy gradient as these parameters are varied.  Since we have exact analytical solutions we do not need to collect data from agents interacting with the environment, which removes the noise of empirically simulating to determine an algorithms performance.



\section{The Gambler's ruin POMDP and its explicit gradient distribution}
\label{sec:gamblers}

We will now use the distributional approach outlined above to analytically study the evolution of parameters under the \textsc{reinforce} algorithm in a \textsc{pomdp}. To do this, we must design a problem in which the distribution of interest can be derived via combinatorics.

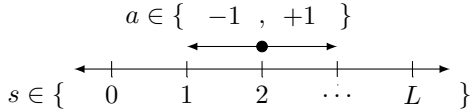
\begin{figure}[hbtp]
    \centering
    \begin{tikzpicture}
\draw[latex-latex] (-1.5, 0) -- (3.5, 0) ;
\foreach \x in {-1, 0, 1, 2, 3}
    \draw[shift={(\x, 0)}, color=black] (0pt, 3pt) -- (0pt, -3pt);
    
\node[align=left] at (-2, -0.35) {$s \in \{$};
\node[align=left] at (3.7, -0.35) {$\}$};
\foreach \x in {0, 1, 2} 
    \draw[shift={(\x-1, 0)}, color=black] (0pt, -3pt) -- (0pt, -3pt) node[below] {$\x$};
\draw[shift={( 2, 0)}, color=black] (0pt, -6pt) -- (0pt, -6pt) node[below] {$\ldots$};
\draw[shift={( 3, 0)}, color=black] (0pt, -3pt) -- (0pt, -3pt) node[below] {$L$};

\filldraw (1, 0.3) circle (2pt);
\draw[latex-latex] (0, 0.3) -- (2, 0.3);
\node[anchor=east,align=right] at (0.1, 0.7) {$a \in \{$};
\node at (0.5, 0.7) {$-1$};
\node at (1, 0.6) {$,$};
\node at (1.5, 0.7) {$+1$};
\node[anchor=west] at (1.9, 0.7) {$\}$};
\end{tikzpicture}
    \caption{Depiction of the 1D gambler's ruin \textsc{pomdp}.}
    \label{fig:gamblers}
    \vskip -\baselineskip
\end{figure}

We define an \textsc{mdp} analogous to the classical gambler's ruin Markov chain, in which the difficulty of exploration can be tuned by changing the starting state, reward function, or number of states in the chain. We then compute an exact solution for the value distribution as a function of these environmental attributes. The gambler's ruin \textsc{pomdp} (see Fig.~\ref{fig:gamblers}) has one starting state $s_0$, a state space $\S = \{0,..,L\}$, two terminating (absorbing) states $\S_{\text{terminating}} = \{0, L\}$, two actions $a \in \{-1, +1\}$, and the reward for every action is $-1$ unless it causes a transition to a terminating state $s$, in which case it gets a reward of $\lambda_{s}$. Finally, the agent is completely ignorant about state information -- it does not know what $s_0$ is or what state it is in.

\subsection{Derivation of the value distribution}
We begin calculating the value distribution for any arbitrary differentiable policy $\pi_{\vec{\theta}}$. Note that in this \textsc{mdp}, we only have one starting state $s_0$, so the return distribution $Z^\pi(s_0) \coloneqq Z^\pi(s_0, \pi(s_0))$ is in fact the episodic return distribution $G$ for an on-policy trajectory: 
$$\Pr{G = g} = \Pr{G = g \mid s_0}= \Pr{Z^\pi(s_0) = g},$$ 
where $\pi(\cdot) = \pi_{\vec{\theta}}(\cdot)$. 
Then, to enhance readability, let 
$$C_{g} = \{(s,t) : s \in \S_\text{terminating}, \ g = \lambda_{s} - t \}$$ 
be the set of $(s,t)$ pairs where $s$ is a terminating state that can give a return of $g$ and $t$ is the associated time-step at which the state $s$ is reached and the return $g$ is obtained. Note that since this \textsc{mdp} only has two terminating states, for each return $g$, $C_{g}$ can contain at most two elements. Now, the PDF can be split as
\begin{equation}\label{gamblers:PDF}
    \Pr{G = g} = \sum_{\mathclap{(s,t) \in C_{g}}} \Pr{G = (\lambda_{s} - t)}.
\end{equation}

\begin{thm}
\label{gamblers:return}
Consider the 1D gambler's ruin \textsc{pomdp} described above, illustrated in Fig.~\ref{fig:gamblers}. Consider a policy $\pi_{\vec{\theta}}$. The episodic return distribution $G$ for an on-policy trajectory satisfies
\begin{equation}
    \Pr{G = g} \!=\! \sum_{\mathclap{(s,t) \in C_{g}}} (1 - p)^{(t - r_{s,t})} p^{r_{s,t}} a_{s_0,L}(t, r_{s, t}),
\end{equation}
where $a_{s_0,L}(t, r_{s, t})$ is the number of paths that end on state $s_t$ after taking $r_{s,t}$ steps to the right: 
\begin{align}
\label{gamblers:count}
    a_{s_0,L}(t,r_{s,t}) &= \sum_{\mathclap{i=-\infty}}^\infty
      \binom{t'}{r' + iL} 
    - \binom{t'}{(r' + s_0) + iL}, 
    \text{ s.t. } t' =  t - 1, \ r' =
    \begin{cases}
        r_{s,t} & \text{if } s_t = 0 \\ 
        r_{s,t} - 1 & \text{if }s_t = L 
    \end{cases} \notag \\ 
    &= \! \frac{4}{L} \sum_{\mathclap{k=0}}^{\floor{(L-1)/2}} 2^{t'}\cos^{t'}\bpar{\frac{\pi k}{L}} \sin(\frac{\pi k}{L}s_0)\sin (\frac{\pi k}{L} (2r_{s,t}' + s_0 - t')).
\end{align}

\end{thm}
\begin{proof}[Proof sketch] 
We compute the probability that a path takes $t$ steps and ends at $s_t$ by breaking it down into two parts: the probability that a policy took $r_{s,t}$ steps to the right and $(t - r_{s,t})$ steps to the left, and the total number of paths that end on $s_t$, after taking $r_{s,t}$ steps to the right, denoted $a_{s_0,L}(t, r_{s,t})$. We have
\begin{equation}
\label{gamblers:probability}
\Pr{G = \lambda_s - t} = (1 - p)^{(t - r_{s,t})} p^{r_{s,t}} a_{s_0,L}(t, r_{s, t}).
\end{equation}
Note that the state $s_t$ is only reachable if $s_t - s_0 + t = 0 \pmod 2$. We can then calculate the number of steps to the right as $ r_{s,t} = \sfrac{1}{2}\,(s - s_0 + t).$ The difficulty now lies in computing $a_{s_0,L}(t, r_{s, t})$ and we use what is called ``the method of images'' in the lattice path counting literature~\cite{Humphreyshistorysurveylattice2010}. The intuition is that it is easy to count the number of paths if there are no terminating barriers: it is just given by Pascal's triangle. We now imagine an image of the agent reflected in a mirror situated at the terminating state. This image faithfully carries out the reflected steps made by the original agent. When the agent hits the terminating state, so does its image, and we imagine that they cancel each other. In this way, the number of paths in the presence of terminating barriers can be expressed as the difference between two coefficients from Pascal's triangle. Of course, this picture is further complicated by the second terminating barrier and the concomitant presence of multiple images. The result is \eqref{gamblers:count} (see Appendix~\ref{app:gamblersmethodofimages} for the proof).
Substituting \eqref{gamblers:probability} into \eqref{gamblers:PDF} then gives the distribution of the return.
\end{proof}
%
\begin{figure*}[htb]
    \centering
    \input{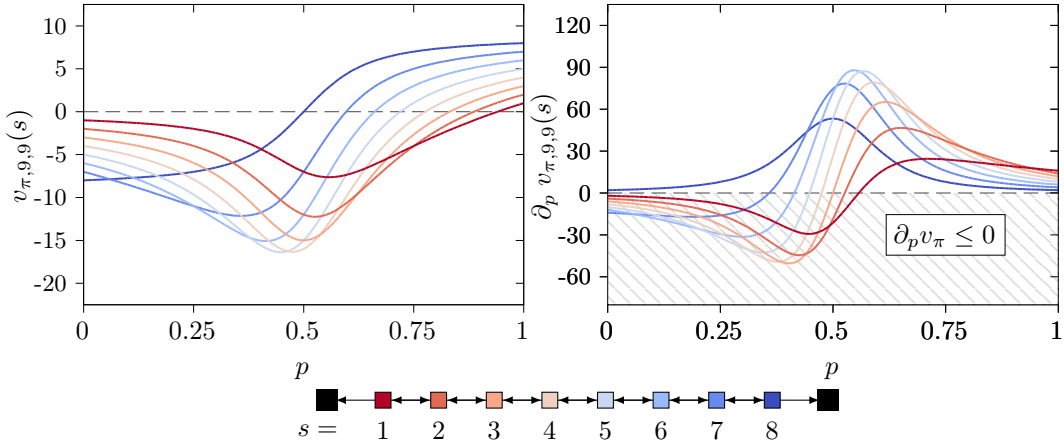}
    \begin{tikzpicture}
        \newcommand\Square[1]{+(-#1,-#1) rectangle +(#1,#1)}
        \def\nshift{0.291in}
        \draw[shift={(-0.52in, 0)}, color=white] (0, 0) -- (0, 0);
        \foreach \x/\c in {%
            8/3b4cc0,7/6889ee,6/9abaff,5/c9d8f0,
            4/edd1c2,3/f7a889,2/e26a53,1/b40426
        } {
            \definecolor{nodecolor}{HTML}{\c}
            \filldraw[shift={(\x*\nshift, 0)}, color=black, fill=nodecolor] \Square{3pt};
            \draw[shift={(\x*\nshift, 0)}, -latex] (3pt, 0) -- (\nshift-3pt, 0);
            \draw[shift={(\x*\nshift, 0)}, -latex] (-3pt, 0) -- (-\nshift+3pt, 0);
            \node[shift={(\x*\nshift, -0.4)}] {\x};
        };
        \fill[shift={(0, 0)}, color=black] \Square{4pt};
        \node[shift={(0, -0.4)}] {$s=\hphantom{s}$};
        \fill[shift={(9*\nshift, 0)}, color=black] \Square{4pt};
    \end{tikzpicture}
    \caption{Policy value function (Eq.~\ref{eq:1d-value-fn}, left), and expected gradient (right) for the 1D gambler's ruin environment ($L = 9$, $\lambda_0 = 0$, $\lambda_L = 9$) as a function of the probability that an agent takes a positive step $p = \pi(a = +1\mid\vec{\theta})$.
    Note that even with only ten states, this environment demonstrates surprising complexity. 
    For example, starting with a purely random policy, it is impossible to guarantee convergence to the global optimum if the starting state is $1$, $2$, or $3$, illustrating the well-known shortfall of local optimization in non-convex settings~\cite{fazelGlobalConvergencePolicy2019,bhandariGlobalOptimalityGuarantees2019}.
    }
    \vskip -0.75\baselineskip
    \label{fig:1d-value-grad}
\end{figure*}
\subsubsection{Value function}
\-\ \ \ \ It is worth noting that one can give the value function -- that is, the expectation of the value distribution -- explicitly in terms of Chebyshev polynomials of the second kind, $U_n(x)$, by using matrix inversion methods~\cite{encinasExplicitInverseTridiagonal2018} (see Appendix~\ref{app:gamblersValue}):
%
\begin{equation}
\label{eq:1d-value-fn}
v_{\pi,L,\lambda}(s) = 
- \sum_{i=0}^{s - 1} I_{i,s - 1,L}(1 - p, p) 
- \sum_{j=s}^{L - 2} I_{s - 1,j,L}(p, 1 - p) (1 - p\lambda \delta_{j,L-2}), 
\end{equation}
where 
\begin{equation*}
    I_{i,j,n}(\alpha, \beta) := 
\alpha^{j - i}
(\alpha\beta)^{\frac{i - j - 1}{2}}
\Big[U_{n - 1}\Big(\frac{1}{2\sqrt{\alpha\beta}}\Big)\Big]^{-1} 
U_i\Big(\frac{1}{2\sqrt{\alpha\beta}}\Big)
U_{n-j-2}\Big(\frac{1}{2\sqrt{\alpha\beta}}\Big).
\end{equation*}
%
\-\ \ \ \ Before looking at the distribution of the gradient explicitly, we can use the policy value function as a first approximation to give us a better intuition of what the 1D gambler's ruin environment is like. Fig.~\ref{fig:1d-value-grad} shows \eqref{eq:1d-value-fn} and its gradient, and we see that even for a relatively small number of states, the environment exhibits surprising features such as non-convexity.

\subsubsection{Gradient distribution}
\label{sec:gradient-distribution}
\-\ \ \ \ Because the policy cannot observe the state for this formulation of this environment, we can express the distribution of the gradient in a straightforward manner  depending on the number of left and right steps taken in a trajectory. The next result is a corollary of Theorem~\ref{gamblers:return}:

\begin{cor}
Let $N_+$ and $N_-$ be respectively the number of left and right steps taken in a trajectory. Then the gradient distribution for a policy $\theta$ is
\begin{equation}
\label{eq:gradient-distribution-1D}
    \nabla_\theta J = \left[\sum_t \nabla_\theta \log(\pi(a_t \mid \vec{\theta}))\right]G  = \Big[ N_- \nabla_{\vec{\theta}} \log(\pi(-1 \mid \vec{\theta})) + N_+ \nabla_{\theta} \log(\pi(+1 \mid \vec{\theta})) \Big]G, 
\end{equation}
where the return is $G = -(N_+ + N_-) + \lambda_{s} = -t + \lambda_{s}$.
\end{cor}

It is \eqref{eq:gradient-distribution-1D} that then allows us to model the evolution of a policy's parameters as a MC, and gain real insight into how learning happens. For example, by constructing a representation of the transition probabilities $\mathcal{P}_{\vec{\theta},\vec{\theta}'}$ for each step of PG, we can look at the probability that an agent will converge to the optimal policy after an infinite number of iterations. Furthermore, extending this analysis to include the effect of a hyperparameter is simple in the MC view (see Fig.~\ref{fig:1d-convergence-vs-step-size}), and would be impossible with a purely sample-based approach.

\begin{figure*}[htbp]
    \centering
    \input{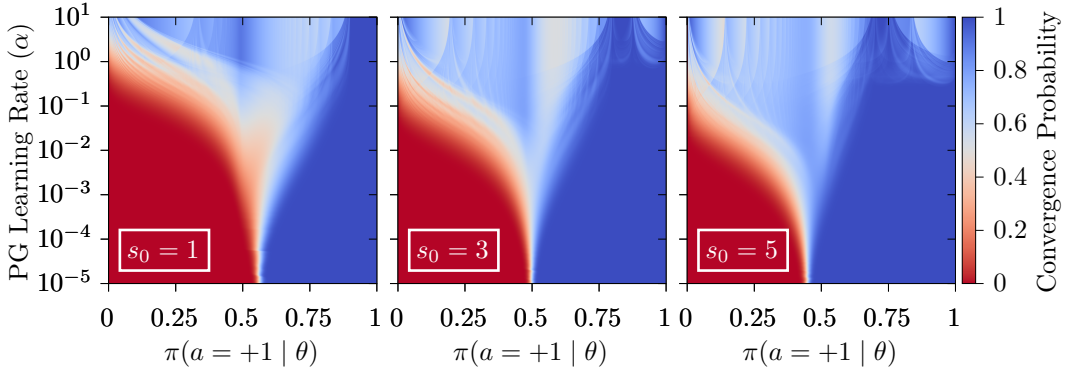}
    \caption{Probability that a Boltzmann PG agent will learn the optimal policy after an \textit{infinite} number of steps of SGA PG on a 10-state 1D gambler's ruin environment (same as in Fig.~\ref{fig:1d-value-grad}).
    Here we show the effect of hyperparameter (here, $\alpha$) tuning on the probability that an agent converges to the optimal policy, for three different starting states. 
    This type of analysis is exactly what would be impossible to calculate using empirical techniques, and is enabled through the use of this Markov chain representation and the exact value distribution.
    For numerical calculation details see Appendix~\ref{app:1d-convergence-vs-step-size}.
    }
    \label{fig:1d-convergence-vs-step-size}
    \vskip -0.75\baselineskip
\end{figure*}

\subsection{Extensions of the environment}

Our method is not only limited to just non-observable environments with state independent policies. We now show that this sort of analysis can be done for policies with state dependence, as well as arbitrarily complex \textsc{mdp}s --- specifically, we treat the case of having arbitrarily many actions.

\subsubsection{State dependence}
\label{sec:state-dependence}

\-\ \ \ \ In this section, we modify the gambler's ruin with observations: suppose the second leftmost state $s=1$ is a `flipped state', where an agent on this state will move left when it intends to go right and vice versa. Furthermore, the \textsc{pomdp} is modified such that the agent observes whether or not it is on the flipped state. Suppose a given trajectory of length $t$ is known to terminate on state $s_t \in \{0, L\}$ after $v_1$ visits to flipped state $s=1$. For a given policy $\vec{\theta}$, we now denote by $p_1$ and $p_2$ the probability of moving right respectively on the flipped state and on all other states.

For $1 \leq i \leq v_1$, let $n_i$ be the timestamp at which the flipped state is visited for the $i$th time. Suppose $s_t = 0$; we can count the number of paths by considering all possibilities of paths starting from $s_0$ and landing on the flipped state for the first time, followed by a path starting on the flipped state and returning to itself $v_1 - 1$ times with a permissible number of steps such that $n_{v_1} = t-1$, and the final transition leads to state $0$. The number of trajectories ending on the leftmost state, denoted $\tau_{s_0, L}(t, v_1, 0)$, letting $\Delta_i = n_i - n_{i-1}$, is
\begin{equation} \label{eq:1}
\begin{aligned}
    \tau_{s_0, L}(t, v_1, 0) &= \sum_{n_1 = s_0 -1}^{t-2v_1+1} a_{s_0, L}\left(n_1, \frac{n_1 - s_0 + 1}{2}\right) \Bigg[ \sum_{\substack{n_j, 2 \leq j \leq v_1 \\ n_j = n_{j-1} + 2}}^{t - 2(v_1 - n_{j-1})+1} \prod_{i=2}^{v_1} a_{1, L}\left(\Delta_i, \frac{\Delta_i}{2}\right)\Bigg].
\end{aligned}
\end{equation}
%
If $s_t = L$, the number of trajectories $\tau_{s_0, L}(t, v_1, L)$ can be calculated similarly:
\begin{align}
    \label{eq:2}
        \tau_{s_0, L} (t, v_1, L) &=\!\!\!\! \sum_{n_1 = s_0 -1}^{t-2(v_1-1)-(L-1)} a_{s_0, L}\left(n_1, \frac{n_1 - s_0 + 1}{2}\right) \notag \\
        &\quad \Bigg[ \sum_{\substack{n_j, 2 \leq j \leq v_1 \\ n_j = n_{j-1} + 2}}^{t - 2(v_1 - n_{j-1} - 1)-(L-1)} \left(\prod_{i=2}^{v_1} a_{1, L}\left(\Delta_i, \frac{\Delta_i}{2}\right)\right) a_{1, L}\left(t-n_{v_1}, \frac{L+t-n_{v_1}-1}{2}\right)\Bigg].
\end{align}
From \eqref{eq:1} and \eqref{eq:2}, summing over the possible number of visitations $v_1$ to the flipped state, we have
\begin{align}
    \Pr\left\{G = g \right\} \! &=\! \sum_{v_1 = 0}^\infty  \sum_{(s,t) \in C_{g}}  p_1^{\vphantom{1}} p_2^{r_{s,t}-v_1}(1-p_1)^{v_1} (1-p_2)^{t-r_{s,t}-1} \tau_{s_0, L}(t, v_1, s),
\end{align}
noting that there are $v_1$ instances where an action right with probability $p_1$ inadvertently caused an action left. With the number of left and right steps on the flipped state being distinguishable, we have two vectors: $\vec{N}_+ = {[N_{+,f}, N_{+,r}]}^{\intercal}$, $\vec{N}_- = {[N_{-,f}, N_{-,r}]}^{\intercal}$, recording the number of right and left steps taken respectively on the flipped state and on the regular states on a trajectory. With the parameter vector $\vec{\theta} = {[\theta_f, \theta_r]}^{\intercal}$, the distribution of the gradient is:
\[
    \nabla_{\vec{\theta}} J = \Big[ \vec{N}_- \nabla_{\vec{\theta}} \log(\pi(-1 | \vec{\theta})) + \vec{N}_+ \nabla_{\vec{\theta}} \log(\pi(+1 | \vec{\theta})) \Big]G 
\]
Although we only consider adding a single state observation, these calculations can be extended to keeping track of the number of visitations to multiple states of interest.

\section{Higher dimensional examples}
\label{sec:n-dimensional}
\no
\-\ \ \ \ Here, we derive the general solution for the gambler's ruin in an arbitrary number of dimensions, \textit{i.e.} an arbitrary number of actions, and even an arbitrary number of terminating states which each can give different rewards for being reached. To formalize our combinatorial approach, we look to the field of lattice path enumeration, and specifically that of random walks inside the alcoves of affine Weyl groups~\cite{GesselRandomwalkWeyl1992}. These are most simply described as particular $n$-dimensional regions in Euclidean space bounded by hyperplanes~\cite{KrattenthalerLatticePathEnumeration2017}.

\begin{figure}[hbtp]
    \centering
    \vskip -\baselineskip
\pgfplotstableread{%
1 0 2
2 0 1
2 0 2
2 0 3
1 0 3
1 0 4
}\table
\pgfplotstablegetrowsof{\table}
\pgfmathsetmacro{\rows}{\pgfplotsretval-2}
\begin{tikzpicture}[scale=0.85]
\definecolor{xred}{HTML}{b40426}
\definecolor{xblue}{HTML}{3b4cc0}
\pgfplotsset{
    every outer x axis line/.style={draw opacity=0},
    every outer y axis line/.style={draw opacity=0},
    every outer z axis line/.style={draw opacity=0},
    every x tick/.style={black!20,thin},
    every y tick/.style={black!20,thin},
    every z tick/.style={black!20,thin},
    major grid style={solid,black!20,thin},
}
\newcommand{\ternaryrow}[1]{%
\pgfplotstablegetelem{#1}{[index]0}\of\table
\let\x\pgfplotsretval
\pgfplotstablegetelem{#1}{[index]1}\of\table 
\let\y\pgfplotsretval
\pgfplotstablegetelem{#1}{[index]2}\of\table 
\let\z\pgfplotsretval%
}
\begin{ternaryaxis}[
    clip=false,
    ternary limits relative=false,
    major tick length=10,
    xmin=-1,xmax=8,
    ymin=-1,ymax=8,
    zmin=-1,zmax=8,
    xtick distance=1,
    ytick distance=1,
    ztick distance=1,
    xtickmin=1,xtickmax=6,
    ytickmin=1,ytickmax=6,
    ztickmin=1,ztickmax=6,
    xticklabels={,,},
    yticklabels={,,},
    zticklabels={,,}
]
    \foreach \k in {0,...,\rows}
    {
        \ternaryrow{\k}
        \let\X\x
        \let\Y\y
        \let\Z\z
        \pgfmathsetmacro{\K}{\k+1}
        \ternaryrow{\K}
        \addplot3[xblue,-latex,thick] coordinates {(\X,\Y,\Z) (\x,\y,\z)};
    }
    
    \fill[fill=xblue] (axis cs:1,3) circle (2pt);
    \node[anchor=north east] at (axis cs:1,3) {{\scriptsize $\mathbf{(3\; 1\; 0)}$}};
    \node[anchor=north] at (axis cs:0,2) {{\scriptsize $\;\mathbf{(6\; 2\; 2) \in \mathcal{H}^{3}_6}$}};
    

    \addplot3[xred,thick,dashed] coordinates {(1, 0, 4) (0, 0, 4)};
    \node at (axis cs:0,2) {\rotatebox{60}{\textcolor{xred}{$\boldsymbol{\times}$}}};
    \draw[latex-latex] (axis cs:-1,0) -- (axis cs:7,0);
    \draw[latex-latex] (axis cs:0,-1) -- (axis cs:0,7);
    \draw[latex-latex] (axis cs:-1,7) -- (axis cs:7,-1);
    \draw[-latex] (axis cs:6,3) -- (axis cs:6,2) node[anchor=south] {$\hat{x}$};
    \draw[-latex] (axis cs:6,3) -- (axis cs:7,3) node[anchor=north east] {$\hat{y}$};
    \draw[-latex] (axis cs:6,3) -- (axis cs:5,4) node[anchor=north west] {$\hat{z}$};
    \node[anchor=west] at (axis cs:-1,0) {$x = z + 6$};
    \node[anchor=west] at (axis cs:0,-1) {$y = z$};
    \node[anchor=west] at (axis cs:7,-1) {$x = y$};
\end{ternaryaxis}
\end{tikzpicture}
    \vskip -1.5\baselineskip
    \caption{A trajectory in the alcove $\mathcal{D}^3_6$, as defined by \eqref{eq:alcove}, that has been visualized by a projecting onto a triangular grid. Here the path starts at $(3, 1, 0)$ and terminates on the boundary, defined by $y = z$, at $(6, 2, 2)$.}
    \label{fig:alcove-walk}
\end{figure}
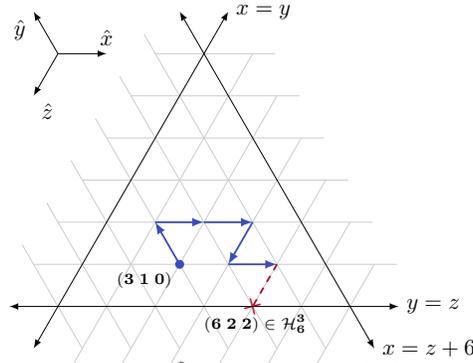

Formally, for dimension $n$ and some positive integer $m$ for the length of the terminating barriers, the environment's action space is the set of $n$ positive unit steps $\vec{e}_i$, where $\vec{e}_i$ is the $i$th standard unit vector, and its state space is the combination of the scaled alcove of the affine Weyl group of type $\tilde{A}_{n-1}$,
\begin{equation}
    \label{eq:alcove}
    \mathcal{D}^n_m\! =\! \{ (x_1, \ldots, x_n) : x_1\!>\!x_2\!>\!\cdots\!>\!x_n\!>\!x_1 - m \}
\end{equation}
and the hyperplanes that comprise its boundary,
\begin{equation}
    \label{eq:hyperplanes}
    \begin{split}
        \mathcal{H}_m^n = &\{(x_1, \ldots, x_n) : x_i - x_{i + 1} = 0, 1 \leq i < n \} \cup \{ (x_1, \ldots, x_n) : x_1 - x_n = m \}. \\
    \end{split}
\end{equation}
Thus for dimension $n$ and integer scaling factor $m$ the state and action spaces are
\begin{equation}
    \label{eq:nd-state-action-space}
    \S = \mathcal{D}^n_m \cup \mathcal{H}_m^n \qq{and} \mathcal{A} = \{\vec{e}_1,\vec{e}_2,\ldots,\vec{e}_n\}.
\end{equation}
Note that all states in $\mathcal{H}_m^n$ are terminating, analogous to the leftmost and rightmost states in 1D gambler's ruin. \latin{E.g.}, in three dimensions, a random walk in the alcove $\mathcal{D}^3_6$ can be visualized as a walk on a triangular lattice in which the path does not touch the axes $x=y$, $y=z$ or $x = z - 6$ (Fig.~\ref{fig:alcove-walk}).

\paragraph{Value distribution}
We can reduce the calculation of the return distribution to that of counting paths. In this case, there is a known result for the number of random walks $b_{\vec{\eta}\vec{\nu}}$ that start at ${\vec{\eta}~\in~\mathcal{D}^n_m}$ and end at ${\vec{\nu} \in \mathcal{D}^n_m}$ that consist of only positive unit steps and stay in the alcove $\mathcal{D}^n_m$. Specifically, $b_{\vec{\eta}\vec{\nu}}$ was originally given by Filaseta~\cite{Filasetanewmethodsolving1985}, although not described as such, and can be written as~\cite{GrabinerRandomWalkAlcove2002}
\begin{equation}
    \label{eq:weyl-paths}
    b_{\vec{\eta}\vec{\nu}} =\!\! \sum_{\sigma \in S_n} \sum_{\sum j_i = 0} \!\!\! \sign(\sigma) \frac{T!}{\prod_{i=1}^n (m j_i + \nu_{\sigma(i)} - \eta_i)!}
\end{equation}
where $\sigma \in S_n$ is a permutation in the symmetric group of order $n$, and ${T = \sum_i \nu_i - \eta_i}$ is the number of steps. In the case that the policy has no knowledge of state, the probability that a policy starting at ${\vec{\eta} \in \mathcal{D}^n_m}$ is at state ${\vec{\nu} \in \mathcal{D}^n_m}$ after $t$ steps is
\begin{equation*}
    \textstyle \Pr{\vec{s}_T = \vec{\nu}}
    = b_{\vec{\eta}\vec{\nu}} \prod_{i=1}^n \pi(\vec{e}_i \mid \vec{\theta})^{\vec{\nu}_i - \vec{\eta}_i}.
\end{equation*}
Note that this is problematic if the ending state is terminating, $\vec{\nu} \in \mathcal{H}_m^n$, since the number of paths from \eqref{eq:weyl-paths} is zero. We address this by noting that only one state from inside $\mathcal{D}^n_m$ can reach $\vec{\nu}$ in one step (\latin{e.g.}, $x_3 - x_4 = 0$ can only be reached if the last step was $\vec{e}_3$). We can then write the probability that a trajectory terminates at ${\vec{\nu} \in \mathcal{H}^n_m}$, where $\vec{\nu}^{\,\prime} \in \mathcal{D}^n_m$ is its only possible neighboring state, as
\begin{equation}
    \label{eq:nd-term-prob}
        \Pr{\vec{s}_T = \vec{\nu} } = 
    \Pr\{\vec{s}_{T-1} = \vec{\nu}^{\,\prime} \} \pi(a = \vec{\nu} - \vec{\nu}^{\,\prime}\mid\vec{\theta}).
\end{equation}
Extending from gambler's ruin (viewed as the special case $n = 2$), we define a constant reward of $-1$ per step and a bonus reward $\lambda_{\vec{s}}$ for transitioning to a (terminating) state that lies on the hyperplanes that bound the alcove \eqref{eq:hyperplanes}. Thus, the total return for a \textit{particular} trajectory that starts at ${\vec{\eta} \in \mathcal{D}^n_m}$ and terminates at ${\vec{\nu} \in \mathcal{H}^n_m}$ is
\begin{equation*}
g_{\vec{\eta} \vec{\nu}} = \lambda_{\vec{\nu}} - \sum_{i=1}^n \nu_i - \eta_i,
\end{equation*}
and combining this with \eqref{eq:nd-term-prob}, we can write value distribution as follows.
\begin{thm} \label{weyl:return}
Consider the affine Weyl group $\mathcal{D}^n_m$ with hyperplanes $\mathcal{H}^n_m$ and let $\pi_{\vec{\theta}}$ be a policy. Then the episodic return distribution $G$ for an on-policy trajectory satisfies
\begin{equation}
    \Pr{G = g \mid s_0 = \vec{\eta}} = \!\! \sum_{{\vec{\nu} \in \mathcal{H}^n_m}} \!\! \mathds{1}_{g_{\vec{\eta} \vec{\nu}} = g} \Pr{\vec{s}_T = \vec{\nu}}.
\end{equation}
\end{thm}

\paragraph{Gradient distribution} 
Finally, we note that for a policy that cannot observe the state, we can calculate the gradient of the policy value function with respect to the parameters $\vec{\theta}$ from Theorem~\ref{weyl:return}.
\begin{cor}
The distribution of the gradient for a given policy $\vec{\theta}$ on the affine Weyl group is
\[
    \nabla_{\theta} J = \left[\sum_{i=1}^n (\nu_i - \eta_i)\, \nabla \log{}(\pi(a = \vec{e}_i \mid \vec{\theta}))\right]\times G_{\vec{\eta} \vec{\nu}}.
\]
\end{cor}

\section{Related work}
\no
\label{sec:related}
\-\ \ Most existing theoretical results on PG algorithms prove asymptotic guarantees of convergence to a stationary point, or a local maximum, using convexity~\cite{borkarMethodConvergenceStochastic2000, zhangGlobalConvergencePolicy2019}. Recent works also include theoretical proofs of convergence to the global optimum in certain settings, using results from control theory~\cite{fazelGlobalConvergencePolicy2019, bhandariGlobalOptimalityGuarantees2019} and non-convex optimization~\cite{agarwalOptimalityApproximationPolicy2019}. New policy gradient methods have also been proposed with tighter bounds for convergence and sample complexity~\cite{xuSampleEfficientPolicy2020}.

The distributional perspective of RL, where one models the distribution of returns rather than only estimating the expected value, was recently proposed by Bellemare \latin{et.\ al.}~\cite{bellemareDistributionalPerspectiveReinforcement2017,morimuraParametricReturnDensity2012}. Applications of this view to algorithm design have shown great success~\cite{bellemareDistributionalPerspectiveReinforcement2017}. Follow-up work such as quantile and implicit quantile regression~\cite{dabneyDistributionalReinforcementLearning2018,dabneyImplicitQuantileNetworks2018}, exploration algorithms~\cite{tangExplorationDistributionalReinforcement2018}, and distributional versions of actor critic methods~\cite{barth-maronDistributedDistributionalDeterministic2018,tesslerDistributionalPolicyOptimization2019} have further demonstrated the empirical utility of the distributional view in deep RL. Theoretical proofs have been offered for the convergence of distributional algorithms using various methods such as classical stochastic approximation proof techniques~\cite{rowlandAnalysisCategoricalDistributional2018}; and the distributional view has also been used to prove convergence of regular RL algorithms~\cite{lyleComparativeAnalysisExpected2019,amortilaDistributionalAnalysisSamplingBased2020}. 

Identifying effective learning strategies for \textsc{pomdp}s has shown to be difficult, and early works~\cite{SinghLearningStateEstimationPartially1994,pendrithAnalysisDirectReinforcement1998} give reasons why conventional RL is inadequate, particularly citing the lack of desired convergence properties. Several empirically successful approaches used dynamic programming~\cite{bagnellPolicySearchDynamic2004} or reduced the \textsc{pomdp} to an equivalent deterministic model~\cite{ngPEGASUSPolicySearch2013}, but the performance of these methods is restricted by the choice of  policy space. 
Early work by Bartlett and Baxter~\cite{bartlettEstimationApproximationBounds2002,baxterReinforcementLearningPOMDP2000,baxterInfiniteHorizonPolicyGradientEstimation2001} considered the underlying state of the \textsc{pomdp} evolving as a Markov chain, for a fixed policy; they then used this view to develop an algorithm to approximate the performance gradient, with accuracy linked to the mixing time of the MC.

Finally, relatively few works view the learning process itself as a MC, or use this perspective to develop convergence proofs for existing algorithms. Yu establishes the convergence of averaged iterates for fixed step-size algorithms by considering them as weak Feller MCs~\cite{yuWeakConvergenceProperties2017}. Amortila \latin{et.\ al.}\ consider RL algorithms whose update rules only depend on the current state and a sampled transition as inducing MCs on the space of value functions, and develop a general proof strategy that establishes convergence for $TD(\lambda)$, Q-Learning, and other algorithms~\cite{amortilaDistributionalAnalysisSamplingBased2020}.

\section{Conclusion and future work}

In this paper, we built on the idea of viewing the policy gradient learning process as a Markov chain which can be studied by using tools from distributional RL. This approach can be used to gain understanding of the learning process beyond general proofs of convergence, such as explicit probabilities that an agent learns the optimal policy. We employed this novel methodology for the probabilistic analysis of policy gradient algorithms by designing a class of \textsc{pomdp}s and showing analytically how an agent's convergence to the optimal policy changes as a function of the learning rate and initial parameters. We also demonstrate how the distribution of learnable policies is changed by using a different optimizer, e.g. SGD with momentum. All of these analyses are impossible to perform without exact solutions due to the computational intractability of sampling. The \textsc{pomdp}s we derived and analyzed are novel in their own right. 

The concept of \textsc{mdp}s where the value distribution can be computed analytically, rather than from empirical simulations, is new and can potentially benefit the reproducibility of RL algorithms, for example, by making it possible to identify whether or not random seeds were used that are cherry-picked~\cite{hendersonDeepReinforcementLearning2018, bouthillierUnreproducibleResearchReproducible2019}. The introduction of exactly solvable models into RL could provide many benefits, from assisting with algorithm design, to discovering pathological behaviour. We focused on (intuitive) discrete-state problems because of the vast literature on random walks on such structures, but even intuitively simple problems can provide important insight. For instance, Reddi \latin{et.\ al.}~\cite{reddiConvergenceAdam2019} used a very simple model to find that despite abundant empirical evidence, Adam does not always converge~\cite{kingmaAdamMethodStochastic2015}. We hope that simple examples with analytical solutions can provide similar advances in the understanding of RL algorithms. Although we only investigated SGD and SGD with momentum, future work could investigate how other optimizers affect the distribution of learnable policies. Further utility could be provided by extending the gambler's ruin and its n-dimensional solution to \textsc{pomdp}s with continuous states. This would be analogous to the Brownian motion of a particle in n-dimensions. 

Our distributional perspective also allows for the analysis of distributional RL algorithms vs expected RL algorithms. In a version of a 2D gambler's ruin environment  with some state dependence, one could investigate questions like ``is the distributional RL update rule equivalent to the expected RL update rule when non-linear function approximation is used''. More broadly, algorithms of interest could be studied analytically on environments that contain features of interest like catastrophic failures, even in limiting regimes that are too large-scale for simulations.

\section*{Acknowledgements}
\no 
We thank Luc Devroye for assisting us with the literature review on path counting as well as putting up with regular invasions of his office.

\bibliography{references}
\bibliographystyle{abbrv}

\newpage

\medskip

\appendix

\setcounter{equation}{0}
\setcounter{section}{0}
\setcounter{figure}{0}
\setcounter{table}{0}
\setcounter{page}{1}
\makeatletter
\renewcommand{\theequation}{A\arabic{equation}}
\renewcommand{\thefigure}{A\arabic{figure}}

\makeatletter
\newcommand\binomialCoefficient[2]{%
    \c@pgf@counta=#1
    \c@pgf@countb=#2
    %
    \c@pgf@countc=\c@pgf@counta%
    \advance\c@pgf@countc by-\c@pgf@countb%
    \ifnum\c@pgf@countb>\c@pgf@countc%
        \c@pgf@countb=\c@pgf@countc%
    \fi%
    %
    \c@pgf@countc=1
    \c@pgf@countd=0
    \pgfmathloop
        \ifnum\c@pgf@countd<\c@pgf@countb%
        \multiply\c@pgf@countc by\c@pgf@counta%
        \advance\c@pgf@counta by-1%
        \advance\c@pgf@countd by1%
        \divide\c@pgf@countc by\c@pgf@countd%
    \repeatpgfmathloop%
    \the\c@pgf@countc%
}
\makeatother


\section{Gambler's Ruin}
\subsection{Counting Paths with the Method of Images}
\label{app:gamblersmethodofimages}
We seek to determine $\Pr{G=g}$ for our MDP defined in Fig.~\ref*{fig:gamblers} given some policy. Recall that we let $C_{g} = \{(s,t) : s \in \S_\text{terminating}, g = \lambda_{s} - t \}$ be the set of $(s,t)$ pairs where $s$ is a terminating state that can give a return of $g$ and $t$ is the associated time-step at which the state $s$ is reached and the return $g$ is obtained. Note that since our MDP only has two terminating states, for each return $g$, $C_{g}$ can contain at most two elements. Recall that we have 
\begin{equation}\label{app-gamblers:PDF}
    \Pr{G = g} = \sum_{\mathclap{(s,t) \in C_{g}}} \Pr{G = (\lambda_{s} - t)}.
\end{equation}
%
We compute the probability $\Pr{G=(\lambda_s - t)}$ that a path takes $t$ steps and ends at $s_t$ by breaking it down into two parts: the probability that a policy took $r_{s,t}$ steps to the right and $(t - r_{s,t})$ steps to the left, and the total number of paths $a_{s_0,L}(t, r_{s,t})$ that end on $s_t$, after taking $r_{s,t}$ steps to the right, for the MDP with $L$ states and starting state $s_0$. We have
\begin{equation}
\label{gamblers:probability-app}
\Pr{G = (\lambda_s - t)} = (1 - p)^{(t - r_{s,t})} p^{r_{s,t}} \times a_{s_0,L}(t, r_{s, t}).
\end{equation}
%

We are only interested in the moment when an episode ends,\textit{ i.e.} the time-step that the agent first reaches an absorbing state. We note that since it is only possible to return to a state with a period of 2, a state $s_t$ is only reachable if $s_t - s_0 + t \equiv 0 \pmod 2$. Using this, we calculate the total number of steps to the right as
\begin{equation*}
r_{s,t} = \frac{1}{2}(s - s_0 + t).
\end{equation*}

Now for $a_{s_0, L}(t, r_{s,t})$, we describe how to apply the method of images to counting the number of possible paths in environments. Since we are currently counting the number of paths that are alive on state $s$ at time $t$, we do not care about rewards and thus the problem can be reduced to considering the random walk on the gambler's ruin Markov chain with $L$ states and starting state $s_0$. 

\begin{figure}[htbp]
    \centering
    \def\N{7}
    \begin{tikzpicture}[x=0.75cm,y=0.5cm, 
      pascal node/.style={}, 
      row node/.style={anchor=west, shift=(180: 2.2)},
      another node/.style={anchor=east, shift=(180: -8.6)}]
      \draw[->] (-5.92,-8) -> (-5.92,.75) node[above] {time step $t$};
      \draw[->] (-5.92,-8) -> (5, -8) node[right] {state $s$};
    \path  
    \foreach \n in {0,...,\N} { 
      (-\N/2-1, -\n) node  [row node/.try]{\n \; --}
        \foreach \k in {0,...,\n}{
            (-\n/2+\k,-\n) node [pascal node/.try] {$\binom{\n}{\k}$ }
        }
    }
    \foreach \k in {-3,...,3} {
        (-\N/2+7/2+\k+.05,-\N-1) node [pascal node/.try, rotate=90] {--}
        (-\N/2+7/2+\k, -\N-1.6) node [pascal node/.try] {\pgfmathtruncatemacro\result{2*\k}{\result}}
    }
    ;
    \end{tikzpicture}
    \caption{The number of paths ending on state $s$ at time step $t$, for a random walk on $\Z$ with no absorbing states, where the starting state is $s_0 = 0$. The rows represent the time-steps, and each column represents a state. Note that an empty entry in a row indicates a path count of 0. On time step $t=0$, only state $s_0$ has a trajectory on it and all other states have none (omitted from the figure). On time step $t=1$, only states $s_0 - 1$ and $s_0 + 1$ have trajectories on them.}
    \label{app-fig:pascals}
\end{figure}
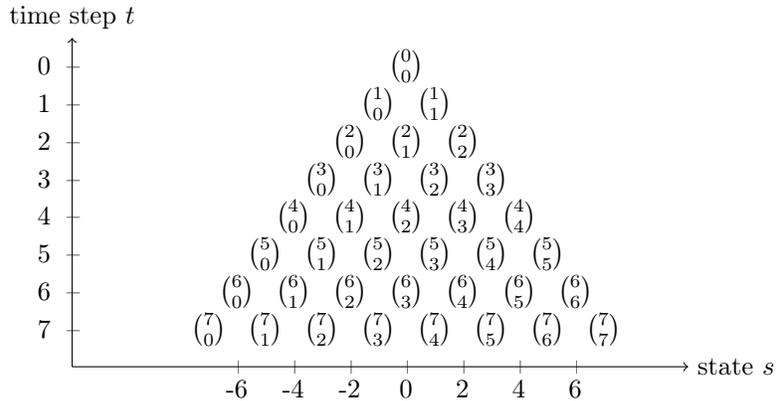

Note that to determine $\Pr{G = g}$, we are only interested in the path counts on the absorbing states, but we provide the path counts for any state for generality. 

First, for illustration, consider the most trivial case, the random walk on $\Z$ with no absorbing states. The evolution of the path count for each state as time increases is simply Pascal's triangle. Glancing at Fig.~\ref{app-fig:pascals} one can see that each entry is of the form $\binom{t}{r_{s,t}}$.

The addition of terminating states to the state space, as is the case in the MDPs we design, complicates the calculation. We inspect the effect of such absorbing barriers in the illustration of the path counts with an example: the random walk on the gambler's ruin Markov chain (Fig.~\ref*{fig:gamblers}) with eight states, $L=7$, and a starting state of three, $s_0 = 3$. We will describe how the method of images is performed on this example by using Fig.~\ref{fig:methodOfImagesBarriers}, hoping that this visualization aids understanding. Note that for $t \geq 1$, we can recursively count the paths:
\begin{equation}
\label{app:state-counts}
    a_{s_0, L}(t, r_{s,t}) = \begin{cases}
      a_{s_0, L}(t-1, r_{s,t}) & s=0,1 \\
      a_{s_0, L}(t-1, r_{s,t}-1) & s = L, L-1 \\
      a_{s_0, L}(t-1, r_{s,t}-1) + a_{s_0, L}(t-1, r_{s,t}) & \text{ otherwise}.
   \end{cases}
\end{equation}
Note that in our figures (Fig.~\ref{app-fig:pascals} and Fig.~\ref{fig:methodOfImagesBarriers}), for an entry $(t, r_{s,t})$, the entries $(t-1, r_{s,t})$ and $(t-1, r_{s,t} - 1)$ correspond respectively to the entries at the top right and top left of $(t, r_{s,t})$. Thus, for $1 < s < L-1$, \eqref{app:state-counts} corresponds to the same propagation as in Pascal's triangle, where each entry takes the value of the sum of the two values immediately above it. For the two leftmost and rightmost states, the path count of an entry is the same as the value of the entry at its top right. This \textit{border effect} can be visualized in Fig.~\ref{fig:methodOfImagesBarriers}, and can be analyzed using the method of images, which we illustrate in Fig.~\ref{fig:methodOfImagesBarriers}(right) by plotting 
\[ \mathrm{entry}[s,t] = \begin{cases} 
0 & \text{ if } s = 0, L \\ 
a_{s_0, L}(t, r_{s,t}) & \text{ if } s \in \{1, \dots, L-1\}
\end{cases}\]
instead of the path counts for every state. This is the number of \textit{alive paths} at $(s,t)$, \latin{i.e.} paths that have not yet been absorbed. This value is only different from the true path count on absorbing states, at which any path dies. We can retrieve the path count $a_{s_0, L}(t, r_{s,t})$ for an absorbing state $(0, t)$ or $(L,t)$ to be respectively the value of the entry at its top right $(t-1, 1)$ or top left $(t-1, L-1)$. 



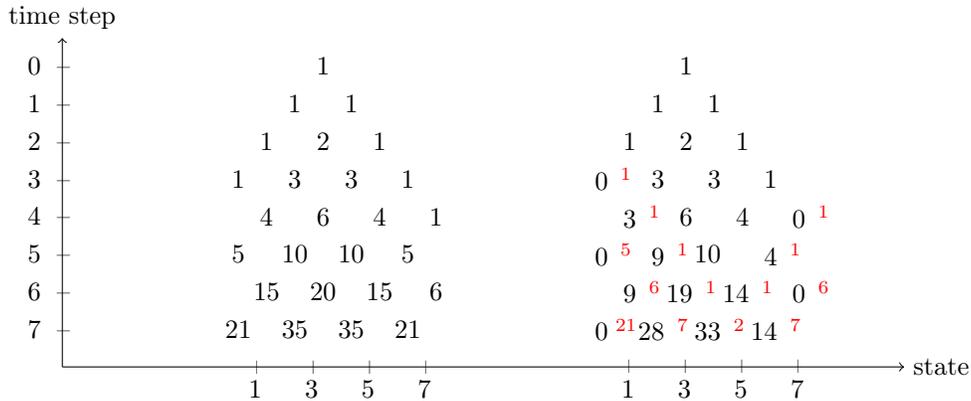
\begin{figure}[h]
    \centering
    \def\N{7}
    \begin{tikzpicture}[x=0.75cm,y=0.5cm, 
      pascal node/.style={}, 
      row node/.style={anchor=west, shift=(180: 1.2)},
      another node/.style={anchor=east, shift=(180: -7.6)}]
      \draw[->] (-3.92,-8) -> (-3.92,.75) node[above] {time step};
      \draw[->] (-3.92,-8) -> (11, -8) node[right] {state};
    \path  
        \foreach \n in {0,...,\N} { 
            (-\N/2, -\n) node  [row node/.try]{\n \; --}
            \foreach \k in {0,...,\n}{
                (-\n/2+\k,-\n) node [pascal node/.try] {
                \global\def\L{7}
                \global\def\s_0{3}
                \pgfkeys{/pgf/fpu}
                \pgfmathparse{
                    binom(\n,\k) 
                        - binom(\n, \k + \s_0)
                        + binom(\n, \k - \L)
                        - binom(\n, \k - \L + \s_0)
                        + binom(\n, \k + \L)
                        - binom(\n, \k + \L + \s_0)
                }
                \pgfmathfloattoint{\pgfmathresult}
                \ifthenelse{\pgfmathresult > -1 }{
                    \let\result\pgfmathresult
                    \pgfmathparse{binom(\n,\k)}
                    \pgfmathfloattoint{\pgfmathresult}
                    \let\pascals\pgfmathresult
                    \pgfmathparse{\result - \pascals}
                    \pgfmathfloattoint{\pgfmathresult}
                    \ifthenelse{0 = 0}{$\pascals$}{
                    $\result^{\color{red}\pgfmathresult}$}
                }{
                }
            }}}
    \foreach \k in {-1,...,2} {
        (-\N/2+8/2+\k+.05,-\N-1) node [pascal node/.try, rotate=90] {--}
        (-\N/2+8/2+\k, -\N-1.6) node [pascal node/.try] {\pgfmathtruncatemacro\result{2*\k + 3}{\result}}
    }
    \foreach \n in {0,...,\N} { 
            \foreach \k in {0,...,\n}{
            (-\n/2+\k,-\n) node [another node/.try] {
            \global\def\L{7}
            \global\def\s_0{3}
            \pgfkeys{/pgf/fpu}
            \pgfmathparse{
                binom(\n,\k) 
                    - binom(\n, \k + \s_0)
                    + binom(\n, \k - \L)
                    - binom(\n, \k - \L + \s_0)
                    + binom(\n, \k + \L)
                    - binom(\n, \k + \L + \s_0)
            }
            \pgfmathfloattoint{\pgfmathresult}
            \ifthenelse{\pgfmathresult > -1 }{
                \let\result\pgfmathresult
                \pgfmathparse{binom(\n,\k)}
                \pgfmathfloattoint{\pgfmathresult}
                \let\pascals\pgfmathresult
                \pgfmathparse{\pascals - \result}
                \pgfmathfloattoint{\pgfmathresult}
                \ifthenelse{\pgfmathresult = 0}{$\result^{\mathclap{\text{\ }}}$}{
                $\result^{\mathclap{\text{\ \ \ \ \ }\color{red}\pgfmathresult}}$}
            }{
            }
        }}}
    \foreach \k in {-1,...,2} {
        (-\N/2+6/2+\k+.05,-\N-0.56) node [another node/.try, rotate=90] {--}
        (-\N/2+3.29+\k, -\N-1.6) node [another node/.try] {\pgfmathtruncatemacro\result{2*\k + 3}{\result}}
    }
    ;
    \end{tikzpicture}
    \caption{This figure is for the underlying gambler's ruin MDP with $L=7$ (see Fig.~\ref*{fig:gamblers}). It provides an intuitive visual understanding of how the method of images acquires the desired result. On the left, we have a truncation of the usual Pascal's triangle. On the right, we have the total number of paths alive ending on each state at time step $t$, with the red exponent representing the difference between the regular Pascal's triangle and the number of paths alive when two absorbing states are introduced.}
    \label{fig:methodOfImagesBarriers}
\end{figure}

Thus, Fig.~\ref{fig:methodOfImagesBarriers}(right) shows the number of alive paths ending on each state for time $0 \leq t \leq 7$ on the random walk with two absorbing barriers. The red exponent represents the difference between the path count with two absorbing states and the path count with no absorbing states. Inspecting the right half of Fig.~\ref{fig:methodOfImagesBarriers}, we note that the exponents in red for the absorbing states (states $0$ and $7$), map directly to the values in Pascal's triangle displayed at the same coordinate on the left half of the figure. This is the core idea of the method of images. The correct path counts are yielded by creating images of Pascal's triangle that are shifted such that subtracting off the images will constrain the absorbing states to always have $0$ paths alive on them. These negative images need to be cancelled out once they reach an absorbing state that is not the one that they were created to constrain. This is done by reflecting the shifted negative image off the other absorbing state, and creating a positive image. Thus, when the negative image reaches the other absorbing state, an equal in magnitude, but positive image will cancel it out, once more keeping the constraint satisfied. Of course, these positive images must be dealt with similarly and cancelled out when they reach the absorbing state that is not the one they just cancelled out the negative image on. This process of adding and subtracting off images must repeat to infinity. This is why the upcoming formula contains an infinite number of binomials before it is simplified.

Using the method of images as previously described, we begin deriving the general formula for counting the number of paths $a_{s_0,L}(t,r_{s,t})$ in this same gambler's ruin environment with length $L$, starting state $s_0$, and two absorbing states, \textit{for all time steps} $t \geq 0$.
For non-absorbing states $0 < s < L$, this is the number of alive paths
\[
\begin{split}
    a_{s_0,L}(t,r_{s, t}) = \binom{t}{r_{s, t}}  
    &- \left[\underbrace{\binom{t}{r_{s,t} + s_0}}_{\text{image of $s_0$ reflected off absorbing state $0$}} + \underbrace{\binom{t}{r_{s,t} - (L - s_0)}}_{\text{image of $s_0$ reflected off absorbing state $L-1$}}\right] \\
    \qquad \qquad
    & + \left[\underbrace{\binom{t}{r_{s,t} + s_0 + (L - s_0)}}_{\text{image of (image of $s_0$ off st. $0$) off state $L-1$}} \! + \! \underbrace{\binom{t}{r_{s,t} - (L - s_0) - s_0}}_{\text{image of (image of $s_0$ off st. $L-1$) off state $0$}}\right] \\
    \qquad \qquad
    & - \left[\binom{t}{r_{s,t} + 2s_0 + (L - s_0)} + \binom{t}{r_{s,t} - 2(L - s_0) - s_0}\right] \\
    \qquad \qquad 
    & + \left[\binom{t}{r_{s,t} + 2s_0 + 2(L - s_0)} + \binom{t}{r_{s,t} - 2(L - s_0) - 2s_0}\right] \\
    \qquad \qquad 
    & + \dots \\
    = \binom{t}{r_{s,t}} 
    & - \Bigg[\sum_{k=0}^\infty \binom{t}{r_{s,t} + s_0 + (s_0 + (L - s_0))k}  + \binom{t}{r_{s,t} - (L - s_0) - (s_0 + (L - s_0))k} \Bigg] \\
    \qquad \qquad 
    & + \left[\sum_{k=1}^\infty \binom{t}{r_{s,t} + (s_0 + (L - s_0))k} + \binom{t}{r_{s,t} - (s_0 + (L - s_0))k}\right], \\
\end{split}
\]
which we can further simplify by combining $\binom{t}{r_s,t}$ with the last sum to yield
\begin{align*}
    a_{s_0,L}(t,r_{s, t}) &= \sum_{k=-\infty}^\infty \binom{t}{r_{s,t} + (s_0 + (L - s_0))k} - \Bigg[ \sum_{k=0}^\infty \binom{t}{r_{s,t} + s_0 + k s_0 + (L - s_0)k} \\
    & \hspace{5.9cm} + \binom{t}{r_{s,t} -ks_0 +(L - s_0)(-k-1)} \Bigg] \\
    &= \sum_{k=-\infty}^\infty \left[
      \binom{t}{r_{s,t} + (s_0 + (L - s_0))k} 
    - \binom{t}{(r_{s,t} + s_0) + (s_0 + (L - s_0))k} 
    \right].
\end{align*} 
We finally get that for absorbing states $s_t \in \{0, L\}$, 
\begin{align}
    a_{s_0,L}(t,r_{s,t}) &= \sum_{k=-\infty}^\infty \left[ \binom{t'}{r'+kL} - \binom{t'}{(r'+s_0) + kL} \right],  \label{eq:paperformula}
\end{align}
where $t' = t-1$ and $r' = \begin{cases}r_{s,t} & \text{ if } s_t = 0 \\ r_{s,t}-1 & \text{ if } s_t = L
\end{cases}$. 

Further simplifications can be performed to make the infinite sum finite. 
First, we can cut off \eqref{eq:paperformula} for values of $k > t'$, where we have have $kL > t'$ and both binomial coefficients will evaluate to zero: 
\begin{align} \label{eq:cutoffvers}
    a_{s_0, L} (t,r_{s,t}) &= \sum_{k=-\infty}^{t'} \left[ \binom{t'}{r'+kL} - \binom{t'}{(r'+s_0) + kL} \right] \nonumber \\ 
    &= \sum_{k=0}^\infty \left[ \binom{t'}{r'-(k-t')L} - \binom{t'}{(r'+s_0) - (k-t')L} \right].
\end{align}

From here, it is possible to simplify this equation further by noting that $\binom{n}{k}$ is the coefficient of $z^{-1}$ in the Laurent series 
\[\frac{(1+z)^n}{z^{k+1}}.\]
Thus, \eqref{eq:cutoffvers} is the coefficient of $z^{-1}$ in the expansion of the function 
\begin{align*}
    f(z) &\ceq \sum_{k=0}^\infty \left[ \frac{(1+z)^{t'} z^{kL}}{z^{r' + t'L + 1}} - \frac{(1+z)^{t'}z^{kL}}{z^{r'+s_0 + t'L + 1}} \right] = \frac{(1+z)^{t'} (z^{s_0} - 1)}{z^{r' + t'L + s_0 + 1}} \sum_{k=0}^\infty z^{kL}\\ 
    &= \frac{(1+z)^{t'} (z^{s_0} - 1)}{z^{r' + t'L + s_0 + 1} (1+z^L)}
\end{align*}
using the geometric series formula (either as a formal summation, or considering $|z|^L < 1$). 

Now recall that the coefficient of $z^{-1}$ in the expansion of $f(z)$ is the residue of the function at the pole $z=0$, so \eqref{eq:cutoffvers} is $\res(f, 0)$. This happens not to be a simple pole, so we use the residue theorem to relate this to the contour integral of $f(z)$ around a circle of radius $R$ 
\begin{align*}
    - 2\pi i \sum_{p} \res(f, p) = \oint_R f(z) dz,
\end{align*}
where $p$ are all the poles inside the circle of radius $R$. It is standard that as $R \to \infty,$ the integral vanishes since the integrand is of order $O(1/z^2)$. Thus $\res(f,0)$ is the negative of the sum of the other residues. In our case, these are the $L-1$ roots of unity other than 1, i.e. $e^{2\pi i k/L}, \; k = 1, \dots, L-1$. These are all simple poles where the residue is much easier to compute -- indeed, we have that for a function $f(z) = g(z)/h(z)$, at a simple pole $p$ where $h(p) = 0$ and $h'(p) \neq 0$, the residue satisfies
\[\res\Big(f(z) = \frac{g(z)}{h(z)}, p\Big) = \frac{g(p)}{h'(p)} = \frac{(1+p)^{t'} (p^{s_0} - 1) }{ \left[ (r'+t'L + s_0 + 1) - p^L (r'+t'L + s_0 + L + 1) \right] p^{r' + t'L + s_0} }.\]

Using this, \eqref{eq:cutoffvers} becomes
\begin{align*}
    a_{s_0,L}(t,r_{s,t}) &= - \sum_{k=1}^{L-1} \res(f, e^{\frac{2\pi i k}{L}}) = - \sum_{k=1}^{L-1} \frac{\left (1+ e^{\frac{2\pi i k}{L}} \right)^{t'} \left(e^{\frac{2\pi i k}{L}s_0} - 1\right) }{ e^{(1/L)2\pi i k (r'+t'L + s_0)} (-L)  } \\ 
    &= \frac{1}{L} \sum_{k=1}^{L-1} \left (1+ e^{\frac{2\pi i k}{L}} \right)^{t'} \left(e^{\frac{2\pi i k}{L}s_0} - 1\right) e^{- \frac{2\pi i k}{L} (r'+ s_0)} \\ 
    &= \frac{1}{L} \sum_{k=1}^{L-1} e^{\frac{\pi i k}{L} t'} \left( 2 \cos\bpar{\frac{\pi k}{L} } \right)^{t'} e^{-\frac{2\pi i k}{L} r'}\left( 1 - e^{-\frac{2\pi i k}{L} s_0} \right) \\ 
    &= \frac{1}{L} \sum_{k=1}^{L-1} e^{\frac{\pi i k}{L} t'} 2^{t'} \cos^{t'}\bpar{\frac{\pi k}{L} } e^{-\frac{2\pi i k}{L} r'} e^{-\frac{\pi i k}{L} s_0}\left(e^{\frac{\pi i k}{L}s_0} - e^{-\frac{\pi i k}{L}s_0} \right)
\end{align*}
Now, matching the terms in the sum at $k$ and $L-k$ gives us, after some algebra, 
\begin{align}
    a_{s_0, L}(t,r_{s,t}) &= \frac{1}{L}  \sum_{\mathclap{k=1}}^{\floor{(L-1)/2}} \!\! 2^{t'} \cos^{t'}\!\! \bpar{\frac{\pi k}{L} } \left(e^{\frac{\pi i k}{L}s_0} - e^{-\frac{\pi i k}{L}s_0} \right) \! \Big(e^{\frac{\pi i k}{L} t'} e^{\frac{\pi i k}{L} (-2r'-s_0)} \nonumber \\ 
    & \hspace{7.5cm} - e^{-\frac{\pi i k}{L} t' } e^{\frac{\pi i k}{L} (2r'+s_0) } \Big) \nonumber \\ 
    &= \frac{1}{L} \sum_{\mathclap{k=1}}^{\floor{(L-1)/2}} \!\! 2^{t'} \cos^{t'}\!\! \bpar{\frac{\pi k}{L} } \bpar{e^{\frac{\pi i k}{L}s_0} - e^{-\frac{\pi i k}{L}s_0} } \Big(e^{-\frac{i\pi k}{L} (2r'+ s_0 - t')}  \nonumber \\ 
    & \hspace{7.5cm} - e^{\frac{i \pi k}{L} (2r' + s_0 -t') } \Big) \nonumber \\
    &= \frac{4}{L} \sum_{\mathclap{k=1}}^{\floor{(L-1)/2}} \!\! 2^{t'} \cos^{t'}\!\! \bpar{\frac{\pi k}{L} } \sin\bpar{\frac{\pi k}{L}s_0} \sin\bpar{\frac{\pi k}{L}(2r' + s_0 - t')} 
\end{align}

\subsection{Value Function}
\label{app:gamblersValue}
For 1D gambler's ruin, the policy value function, which we denote as $v_{\pi,L,
\lambda}(s)$ can be calculated using a number of ways, but here we start by constructing a recursion relation. Recall that in this environment, the reward for every action is $-1$ unless it causes a transition to a terminating state $s$, in which case it gets a reward of $\lambda_{s}$. For brevity, we define the probability of taking a step to the right as $p = \pi(a = +1 \mid \theta)$ and the expected value at a given state $s$ for our policy as $v_s$. Thus we have
\[
v_s = 
\begin{cases}
\hfill (1 - p)\lambda_0 + p v_{2} - 1 & \text{if }s = 1 \\
\hfill p\lambda_L + (1 - p)v_{L - 2} - 1 & \text{if }s = L - 1 \\
\; \; p v_{s+1} + (1 - p)v_{s - 1} - 1 & \text{otherwise}.
\end{cases}
\]
Rearranging things, we can write
\[
\begin{split}
    v_{1} - p v_{2} &= (1 - p)\lambda_0 - 1 \\
    - (1 - p)v_{L - 2} + v_{L - 1} &= p\lambda_L - 1 \\
    - (1 - p) v_{s - 1} + v_{s} - p v_{s+1}  &= - 1. \\
\end{split}
\]
Considering these $v_s$'s as a vector $\vec{v}$ of size $L+1$, we can write this as the following matrix equation:
\[
\bmqty{
1 & -p \\
-(1-p) & 1 & -p \\
& \ddots & \ddots  & \\
& -(1 - p) & 1 & -p \\
& & -(1 - p) & 1 \\
} \vec{v} = \bmqty{
-1 + (1 - p)\lambda_0\\
-1 \\
\vdots \\
-1 \\
- 1 + p\lambda_L \\
},
\]
where we can note that the matrix on the left hand side is a tridiagonal Toeplitz matrix of the form
\[
J = \bmqty{
  \beta & -\alpha &       0 & \cdots &       0 &       0 \\
-\gamma &   \beta & -\alpha & \cdots &       0 &       0 \\
      0 & -\gamma &   \beta & \cdots &       0 &       0 \\
 \vdots &  \vdots &  \vdots & \ddots &  \vdots &  \vdots \\
      0 &       0 &       0 & \cdots &   \beta & -\alpha \\
      0 &       0 &       0 & \cdots & -\gamma &   \beta \\
}
\]
with $\alpha = p$, $\beta = 1$, and $\gamma = (1 - p)$. Thus here we can use the analytically known inverse given by Encinas and Jim\'{e}nez~\cite{encinasExplicitInverseTridiagonal2018}. As in the main text, writing the $i$-th Chebyshev polynomial of the second kind as $U_i$ and defining
\[
I_{i,j,n}(\alpha, \beta) := 
\alpha^{j - i}
(\alpha\beta)^{\frac{i - j - 1}{2}}
\frac{
    U_i(\frac{1}{2\sqrt{\alpha\beta}})
    U_{n-j-2}(\frac{1}{2\sqrt{\alpha\beta}})
}{
    U_{n - 1}(\frac{1}{2\sqrt{\alpha\beta}})
}
\]
where $i \leq j$, we can write the policy value function as
\begin{equation}
\begin{split}
v_{\pi,L,\lambda}(s) =
&-\sum_{i=0}^{s - 1} I_{i,s - 1,L}(1 - p, p) (1 - (1 - p)\lambda_0 \delta_{i,0})\\
&-\sum_{j=s}^{L - 2} I_{s - 1,j,L}(p, 1 - p) (1 - p\lambda_L \delta_{j,L-2}),
\end{split}
\end{equation}
where $\delta_{ij}$ is the Kronecker delta.

\section{Calculations}
\label{app:calc}
This section contains information about the numerical calculations used to produce the figures in this work. An open source implementation of the algorithms in this paper, primarily written in Rust, is available on github at \url{https://github.com/colin-daniels/lattice_rl} and includes code to reproduce all plots.

\subsection{Overview}
In the analysis of the gambler's ruin environment, we are most concerned about $\mathcal{P}_{\theta, \theta'}$, which tells us how the policy's parameters evolve as it learns. In order to use the analytical results for the value distribution laid out in the main work, we need to determine how $\theta$ is updated based on the value distribution. Our examples for the 1D gambler's ruin use policy gradient and plain stochastic gradient ascent, that is,
\begin{equation}
        \theta' = \theta + \alpha \nabla J(\theta).
\end{equation}

For simplicity, we use the \textsc{reinforce} update rule~\cite{williamsSimpleStatisticalGradientfollowing1992} and a batch size of one. In order to calculate $\mathcal{P}_{\theta, \theta'}$, we also need to specify the policy used. In the one-dimensional case, we only have two actions, and for demonstration purposes we use a Boltzmann policy
\begin{equation}
    \begin{split}
        \pi(a = +1 \mid \theta) &= (1 - \epsilon) \frac{1}{2}(1 + \tanh(\theta / \tau)) + \frac{\epsilon}{2} \\
        \pi(a = -1 \mid \theta) &= 1 - \pi(a = +1 \mid \theta)
    \end{split}
\end{equation}
where $\tau$ is temperature, and $\epsilon$ is an exploration factor. Again for simplicity, we set $\tau = 1$ and $\epsilon = 0$.

The next crucial question is how to represent $\mathcal{P}_{\theta, \theta'}$ such that useful calculations can be carried out. The representation itself depends on whether the calculations are purely analytical in nature, or if, like in the case here, they are numerical. For Fig.~\ref*{fig:distributional-evolution} and Fig.~\ref*{fig:1d-convergence-vs-step-size} we use the simplest and most intuitive form of $\mathcal{P}_{\theta, \theta'}$, a matrix in which individual rows and columns correspond to ranges of $\theta$. Constructing this matrix representation of $\mathcal{P}$ is relatively straightforward and can be done by iterating over the possible gradient values and filling in the appropriate elements of $\mathcal{P}$. We avoid the issue of having an infinite number of possible values by iterating in order from largest probability to smallest, and stopping when the total probability remaining is less than $1\times10^{-12}$. The number of rows and columns in our representation of $\mathcal{P}$ is $1024\times1024$, where discretization error is reduced further by sub-sampling for each bin.

\subsection{Figure \ref*{fig:distributional-evolution}}
\label{app:distributional-evolution}
Once a matrix representation of $\mathcal{P}_{\theta, \theta'}$ is constructed as described above, an initial distribution of $\theta$ can be evolved through time by simply performing a matrix multiplication for each time step. For this figure, the 1D gambler's ruin environment is characterized by $L = 9$, $s_0 = 3$, $\lambda_0 = 0$, and $\lambda_L = 9$. The equivalent SGA step size used is $2\times10^{-4}$. For the plot of SGA with momentum, a momentum value of $0.2$ was used with all other parameters identical to plain SGA. The distribution was created by constructing a 2D transition matrix $\mathcal{P}_{\{\theta,v\}, \{\theta',v'\}}$, which was applied to an initial state with $v = 0$, and for each iteration the distribution was summed over all values of $v$ in order to plot.

\subsection{Figure \ref*{fig:1d-convergence-vs-step-size}}
\label{app:1d-convergence-vs-step-size}
The convergence probabilities shown in Fig.~\ref{fig:1d-convergence-vs-step-size} were calculated by realizing that, in this example, the parameter Markov chain is absorbing. From there, we simply calculate the absorption probabilities, which can be done easily by numerically approximating the infinite power of the matrix representation of $\mathcal{P}_{\theta, \theta'}$. For this figure, the 1D gambler's ruin environment is characterized by $L = 9$, $\lambda_0 = 0$,  $\lambda_L = 9$, and the step size $\alpha$ was iterated over from $10^{-5}$ to $10$.

\small

\end{document}

\end{document}